\documentclass[12pt]{article}

\RequirePackage{amsthm,amsmath,amsfonts,amssymb}

\usepackage{enumitem}
\usepackage{cancel}
\usepackage{times}
\usepackage{bbm}
\usepackage{cite}
\usepackage{color, soul}
\usepackage[paper=a4paper, top=2.5cm, bottom=2.5cm, left=2.5cm, right=2.5cm]{geometry}
\usepackage{wasysym}
\usepackage{hyperref}

\hypersetup{
	colorlinks   = true,
	citecolor    = black,
	linkcolor    = blue
}

\usepackage{wasysym}

\def\qed{\hfill $\vcenter{\hrule height .3mm
\hbox {\vrule width .3mm height 2.1mm \kern 2mm \vrule width .3mm
height 2.1mm} \hrule height .3mm}$ \bigskip}

\def \Sph{\mathbb{S}^{n-1}}
\def \Spd{\mathbb{S}^{d-1}}
\def \RR {\mathbb R}
\def \NN {\mathbb N}
\def \EE {\mathbb E}

\def \PP {\mathbb P}
\def \eps {\varepsilon}
\def \vphi {\varphi}

\def \ff {g}

\newcommand\norm[1]{\left\lVert#1\right\rVert}
\newcommand\inner[2]{\langle#1,#2\rangle}
\newtheorem{theorem}{Theorem}
\newtheorem{lemma}[theorem]{Lemma}

\newtheorem{claim}[theorem]{Claim}

\theoremstyle{definition}
\newtheorem{definition}[theorem]{Definition}
\newtheorem{example}[theorem]{Example}
\theoremstyle{remark}
\newtheorem{remark}[theorem]{Remark}
\newtheorem*{remark*}{Remark}

\long\def\symbolfootnotetext[#1]#2{\begingroup
\def\thefootnote{\fnsymbol{footnote}}\footnotetext[#1]{#2}\endgroup}

\title{Community detection and percolation of information in a geometric setting}
\author{Ronen Eldan\thanks{Supported by a European Research Council Starting Grant (ERC StG) and by an Israel Science Foundation Grant no. 715/16.} \\
	Weizmann Institute
	\and Dan Mikulincer\thanks{Supported by an Azrieli foundation fellowship.} \\
	Weizmann Institute 
	\and Hester Pieters\thanks{Supported by an Israel Science Foundation Grant no. 715/16} \\
	Weizmann Institute
}
\begin{document}
\maketitle

\begin{abstract}
We make the first steps towards generalizing the theory of stochastic block models, in the sparse regime, towards a model where the discrete community structure is replaced by an underlying geometry. We consider a geometric random graph over a homogeneous metric space where the probability of two vertices to be connected is an arbitrary function of the distance. We give sufficient conditions under which the locations can be recovered (up to an isomorphism of the space) in the sparse regime. Moreover, we define a geometric counterpart of the model of flow of information on trees, due to Mossel and Peres, in which one considers a branching random walk on a sphere and the goal is to recover the location of the root based on the locations of leaves. We give some sufficient conditions for percolation and for non-percolation of information in this model.
\end{abstract}
\section{Introduction}
Community detection in large networks is a central task in data science. It is often the case that one gets to observe a large network, the links of which depends on some unknown, underlying community structure. A natural task in this case is to detect and recover this community structure to the best possible accuracy. \\
Perhaps the most well-studied model in this topic is the \emph{stochastic block model} (SBM) where a random graph whose vertex set is composed of several communities, $\{c_1,...,c_k\}$ is generated in a way that every pair of nodes $v,u$ which belong to communities $c(u), c(v)$, will be connected to each other with probability $p = p(c(v),c(u))$, hence with probability that only depends on the respective communities, and otherwise independently. The task is to recover the communities $c(\cdot)$ based on the graph. The (unknown) association of nodes with communities is usually assumed to be random and independent between different nodes. See \cite{Abbe18} for an extensive review of this model.\\
A natural extension of the stochastic block model is the geometric random graph, where the discrete set of communities is replaced by a metric space. More formally, given a metric space $(X,d)$, a function $f: V \to X$ from a vertex set $V$ to the metric space and a function $\vphi:\RR_+ \to [0,1]$, a graph is formed by connecting each pair of vertices $u,v$ independently, with probability 
$$
p(u,v) := \vphi \left ( d(f(u), f(v)) \right ).
$$
This model can sometimes mimic the behavior of real-world networks more accurately than the stochastic block model. For example, a user in a social network may be represented as a point in some linear space in a way that the coordinates correspond to attributes of her personality and her geographic location. The likelihood of two persons being associated with each other in the network will then depend on the proximity of several of these attributes. A flat community structure may therefore be two simplistic to reflect these underlying attributes. \\
Therefore, a natural extension of the theory of stochastic block models would be to understand under what conditions the geometric representation can be recovered by looking at the graph. Our focus is on the case that the metric is defined over a symmetric space, such as the Euclidean sphere in $d$-dimensions. By symmetry, we mean that the probability of two vertices to be connected, given their locations, is invariant under a natural group acting on the space. We are interested in the sparse regime where the expected degrees of the vertices do not converge to infinity with the size of the graph. This is the (arguably) natural and most challenging regime for the stochastic block model.
\subsection{Inference in geometric random graphs}

For the sake of simplicity, in what follows, we will assume that the metric space is the Euclidean sphere, and our main theorems will be formulated in this setting; It will be straightforward, yet technical, to generalize our results to more general symmetric space (see \cite{de2017adaptive} for further discussion on this point).\\
We begin by introducing a model for a random geometric graph. Let $\sigma$ be the uniform probability measure on $\Spd$ and let $\vphi: \Spd \times \Spd \to \RR_+$, be of the form $\vphi(x,y) = f(\inner{x}{y})$ for some $f: [-1,1] \to \RR_+$. Consider $\{X_i\}_{i=1}^n\sim \sigma$, a sequence of independently-sampled vectors. 
 Let $G\left(n, \frac{1}{n} \vphi(X_i, X_j)\right)$ be the inhomogeneous Erd\"{o}s-R\'{e}nyi graph model where edges are formed independently with probability $\frac{1}{n} \vphi(X_i, X_j)$  and let $A$ be the adjacency matrix of a random graph drawn from $G\left(n, \frac{1}{n}\vphi(X_i,X_j)\right)$. We may freely assume that $n$ is large enough so that, necessarily, $\frac{1}{n} \vphi(X_i, X_j) \leq 1$.\\
We now define the task of recovering the latent positions of the sample $\{X_i\}_{i=1}^n$ from an observation of the random geometric graph.
\begin{definition}
	We say that the model is $\eps-\emph{reconstructible}$ if, for all $n$ large enough, there is an algorithm which returns an $n \times n$ matrix $\mathcal{X}$ such that 
	$$
	\frac{1}{n^2} \sum_{i,j} |\mathcal{X}_{i,j} - X_i \cdot X_j|^2 \leq \eps.
	$$
\end{definition}

Remark that, due the symmetry of the model, it is clear that the locations can only be reconstructed up to an orthogonal transformation, which is equivalent to reconstruction of the Gram matrix.\\

In order state the conditions under which we prove the possibility of reconstructions, we need some notation. 
For $\varphi$ as above, define the integral operator $A_\vphi: L^2\left(\Spd\right)\to L^2\left(\Spd\right)$ by
$$A_\vphi(g)(x) = \int\limits_{\Spd} \vphi(x,y)g(y)d\sigma(y).$$
It is standard to show that $A_\vphi$ is a self-adjoint compact operator (see \cite{hirsch2012elements}, for example) and so has a discrete spectrum, except at $0$. By definition, $\vphi$ is invariant to rotations and so $A_\vphi$ commutes with the Laplacian. It follows that the eigenfunctions of $A_\vphi$ are precisely the \emph{spherical harmonics} which we denote by
$\{\psi_i\}_{i=0}^\infty$. Thus, if $\lambda_i(A_\vphi)$ denotes the eigenvalue of $\vphi$ corresponding to $\psi_i$ we have the following identity,
\begin{equation} \label{eq: l2 decomp}
\vphi = \sum\limits_{i=0}^\infty \lambda_i \psi_i \otimes \psi_i.
\end{equation}In particular, $\psi_0 = 1$ and for $i = 1,...,d$, $\psi_i$ are linear functionals such that, for $x,y \in \Spd$,
\begin{equation} \label{eq: inner product}
d \cdot \inner{x}{y} = \sum_{l=1}^d \psi_l(x)\psi_l(y).
\end{equation}

Note that in our notation the eigenvalues are indexed by the spherical harmonics, and are therefore not necessarily in decreasing order. By rotational invariance it must hold that
\begin{equation} \label{eq: lambda equality}
\lambda(\vphi) := \lambda_1=...=\lambda_d.
\end{equation}
Define $\|\varphi\|_\infty = \sup_{x,y} \varphi(x,y)$.
We make the following, arguably natural, assumptions on the function $\vphi$:

\begin{enumerate}[label={A\arabic*.}, ref=A\arabic*]
\item \label{assumption1}
There exist $\delta > 0$ such that $\min_{i\neq 1,\dots, d}|\lambda(\vphi) - \lambda_i|>\delta$.
\item \label{assumption 2}
Reordering the eigenvalues in decreasing order $\lambda_{l_0}\geq \lambda_{l_1}\geq\lambda_{l_2}\geq\dots$ there exists $C > 0$ such that for every $i\geq 0, |\lambda_{l_i}|\leq \frac{C}{(i+1)^2}$. 
\end{enumerate}
\begin{remark}
	For the spectral algorithm we propose, the eigenspace associated to $\lambda(\varphi)$ plays as a special role, since it may encode positions of latent vectors in $\RR^d$. Therefore, to extract meaningful information from the algorithm, it is crucial that this eigenspace is $d$-dimensional; a property that is enforced by Assumption \ref{assumption1}. However, it is not clear what would happen if, say, $\lambda_{d+1}$ is very close to $\lambda_d$. Our analysis. which is based on stability of the eigenspaces to perturbations of the kernel, fails in this case, but it is plausible the the algorithm (or another one) may still succeed. 
	
	Regarding Assumption \ref{assumption 2}. While, it is usually the case that, some integrability of eigenvalues is necessary for the problems we consider, we do not know whether the quadratic decay rate is tight. We have made no efforts to optimize this dependence and are willing to conjecture that there are possible improvements.
\end{remark}

\begin{theorem}
\label{thm: reconstruction bound}
	For every $\eps >0$ there exists a constant $C=C(\eps,d)$, such that the model is $\eps$-reconstructible whenever
	\begin{equation} \label{eq:thm1}
	\min_{i\neq 0,\dots,d}|\lambda_{i}-\lambda(\vphi)|^2 \geq C \| \vphi \|_{\infty}.
	\end{equation}
\end{theorem}

\begin{remark}
Observe that, since the left hand side of condition \eqref{eq:thm1} is $2$-homogeneous, whereas its right hand side is $1$-homogeneous, we have that as long as the left hand side is nonzero, by multiplication of the function $\vphi$ by a large enough constant,  the condition can be made to hold true. 
\end{remark}
Theorem \ref{thm: reconstruction bound} can be compared to the known bounds for recovery in the stochastic block model (see \cite{Abbe18}). In particular, it is illuminating to compare the SBM with two communities to linear kernels in our model. In this case, both models are parameterized by two numbers. In the SBM these are the inter- and intra- communities probabilities and in our model, the coefficients of the linear function. In the SBM, denote the parameters as $a$ and $b$, then recovery of the communities depends on the ratio $\frac{(a-b)^2}{a+b}$. The example below gives a similar result for linear kernels, with a dimensional affect, which typically makes reconstruction easier.
\begin{example}
	Consider the linear kernel,
	$\varphi(x,y) =\gamma + \zeta\langle x,y\rangle$, with $|\zeta|\leq \gamma$. A calculation shows that
	\begin{align*}
	\lambda_0 &= \gamma \\
	\lambda(\vphi) &= \frac{\zeta}{d}.
	\end{align*}
	Applying our theorem, we show that the model is reconstructible whenever
	\[
	\left| \gamma-\frac{\zeta}{d}\right|^2 \geq C \cdot (\gamma+\zeta).
	\]
\end{example}
\paragraph{Methods and related works.} Our reconstruction theorem is based on a spectral method, via the following steps:
\begin{enumerate}
\item 
We observe that by symmetry of our kernel, linear functions are among its eigenfunctions. We show that the kernel matrix (hence the matrix obtained by evaluating the kernel at pairs of the points $(X_i)$) will have respective eigenvalues and eigenvectors which approximate the ones of the continuous kernel. 
\item 
Observing that the kernel matrix is the expectation of the adjacency matrix, we rely on a matrix concentration inequality due to Le-Levina-Vershynin \cite{LeLevVer18} to show that the eigenvalues of the former are close to the ones of the latter.
\item 
We use the Davis-Kahan theorem to show that the corresponding eigenvectors are also close to each other.
\end{enumerate}
The idea in Steps 2 and 3 is not new, and rather standard (see \cite{LeLevVer18} and references therein). Thus, the main technical contribution in proving our upper bound is in Step 1, where we prove a bound for the convergence of eigenvectors of kernel matrices. So far, similar results have only been obtained in the special case that the Kernel is positive-definite, see for instance \cite{braun2006accurate}.\\
The paper \cite{valdivia2018relative} considers kernels satisfying some Sobolev-type hypotheses similar to our assumptions on $\varphi$ (but gives results on the spectrum rather than the eigenvectors). Reconstruction of the eigenspaces has been considered in \cite{SusTanPrie13}, for positive definite kernels in the dense regime; in \cite{SusTanPrie14}, for random dot products graphs; and in \cite{Valdivia_DeCastro19} in the dense and relatively sparse regimes, again for kernels satisfying some Sobolev-type hypotheses. \\
Let us also mention other works which augmented the SBM with geometric information. The paper \cite{deshpande2018contextual} considers the problem of discrete community recovery when presented with informative node covariates, inspired by the spiked  covariance  model. The authors derived a sharp threshold for recovery by introducing a spectral-like algorithm. However, the model is rather different than the one we propose in which the community structure is continuous.\\
A model which is slightly closer to the one we consider appears in \cite{galhotra2017geometric}. In this model, communities are still discrete but the edge connectivity depends continuously on the latent positions of nodes on some $d$ dimensional sphere. In such a model, recovery of the communities may be reduced to more combinatorial arguments. Indeed, the number of common neighbors can serve as an indicator for checking whether two nodes come from the same community. A similar idea was previously explored in \cite{bubeck2016testing}, where a triangle count was used to establish a threshold for detecting latent geometry in random geometric graphs.\\
\subsection{Percolation of geometric information in trees} \label{sec:perco}
The above theorem gives an upper bound for the threshold for reconstruction. The question of finding respective lower bounds, in the stochastic block model, is usually reduced to a related but somewhat simpler model of percolation of information on \emph{trees}. The idea is that in the sparse regime, the neighborhood of each node in the graph is usually a tree, and it can be shown that recovering the community of a specific node based on observation of the entire graph, is more difficult than the recovery of its location based on knowledge of the community association of the leaves of a tree rooted at this node. For a formal derivation of this reduction (in the case of the SBM), we refer to \cite{Abbe18}.\\
This gives rise to the following model, first described in Mossel and Peres \cite{MosPer03} (see also \cite{Mos04}): Consider a $q$-ary tree $T=(V,E)$ of depth $k$, rooted at $r \in V$. Suppose that each node in $V$ is associated with a label $\ell: V \to \{1,..,k\}$ in the following way: The root $r$ is assigned with some label and then, iteratively, each node is assigned with its direct ancestor's label with probability $p$ and with a uniformly picked label with probability $1-p$ (independent between the nodes at each level). The goal is then to detect the assignment of the root based on observation of the leaves.\\
Let us now suggest an extension of this model to the geometric setting. We fix a Markov kernel $\vphi(x,y) = f(\langle x,y \rangle)$, normalized such that $\int_{\Sph} \varphi(x,y) d \sigma(y) = 1$ for all $x \in \Spd$. We define $\ff: T \to \Spd$ in the following way. For the root $r$, $\ff(r)$ is picked according to the uniform measure. Iteratively, given that $\ff(v)$ is already set for all nodes $v$ at the $\ell$-th level, we pick the values $\ff(u)$ for nodes $u$ at the $(\ell+1)$th level independently, so that if $u$ is a direct descendant of $v$, the label $\ff(u)$ is distributed according to the law $\varphi(\ff(v), \cdot) d \sigma$. 

Denote by $T_k \subset V$ the set of nodes at depth $k$, and define by $\mu_k$ the conditional distribution of $\ff(r)$ given $(\ff(v))_{v \in T_k}$. We say that the model has positive information flow if 
$$
\lim_{k \to \infty} \EE \left [ \mathrm{TV} (\mu_k, \sigma) \right ] > 0.
$$
Remark that by symmetry, we have
$$
\EE \left [ \mathrm{TV} (\mu_k, \sigma) \right ] = \EE \left [ \mathrm{TV} (\mu_k, \sigma) | \ff(r) = e_1 \right ]
$$
where $r$ is the root and $e_1$ is the north pole. 

Our second objective in this work is to make the first steps towards understanding under which conditions the model has positive information flow, and in particular, our focus is on providing nontrivial sufficient conditions on $q, \varphi$ for the above limit to be equal to zero. 

Let us first outline a natural sufficient condition for the information flow to be positive which, as we later show, turns out to be sharp in the case of Gaussian kernels. Consider the following simple observable,
$$
Z_k := \frac{1}{|T_k|} \sum_{v \in T_k} \ff(v).
$$
By Bayes' rule, we clearly have that the model has positive information flow if (but not only if)
\begin{equation}\label{eq:flowdecay}
\liminf_{k \to \infty} \frac{ \EE[ \langle Z_k, e_1 \rangle | \ff(r)=e_1] }{\sqrt{\mathrm{Var} \left [\langle Z_k, e_1 \rangle \right | \ff(r) = e_1 ]} } > 0. 
\end{equation}
This gives rise to the parameter
$$
\lambda(\vphi) := \int_{\Spd} \langle x , e_1  \rangle  \varphi(e_1, x) d\sigma(x),
$$
which is the eigenvalue corresponding to linear harmonics. As $\langle x , e_1  \rangle$ is an eigenfunction of the Markov kernel, it is preserved under repeated applications. Coupling this observation with the linearity of expectation, we have
\begin{equation} \label{eq:markovspectrum}
\EE[ \langle Z_k, e_1 \rangle | \ff(r)=e_1] = \lambda(\vphi)^k.
\end{equation}
Moreover, by symmetry if $j \neq 1$,
$$
0 = \int_{\Spd} \langle x , e_j  \rangle  \varphi(e_1, x) d\sigma(x),
$$
from which we may conclude for any $v \in \Spd$,
\begin{equation} \label{eq:markovprojection}
\lambda(\vphi)\langle v, e_1  \rangle = \int_{\Spd} \langle x , v  \rangle  \varphi(e_1, x) d\sigma(x),
\end{equation}
For two nodes $u,v \in T$ define by $c(u,v)$ the deepest common ancestor of $u,v$ and by $\ell(u,v)$ its level. A calculation gives
\begin{align*}
\mathrm{Var} \left [\langle Z_k, e_1 \rangle \right | \ff(r) = e_1 ] & = \frac{1}{q^{2k}} \sum_{u,v \in T_k} \EE \left [ \ff(v)_1 \ff(u)_1  | \ff(r) = e_1 \right ] - \lambda(\vphi)^{2k} \\
& = \frac{1}{q^{2k}} \sum_{u,v \in T_k} \EE \left [ \EE [\ff(v)_1 | \ff(c(u,v))] \EE[\ff(u)_1 | \ff(c(u,v))]  | \ff(r) = e_1 \right ] - \lambda(\vphi)^{2k} \\
& = \frac{1}{q^{2k}} \sum_{u,v \in T_k} \EE \left [\ff(c(u,v))_1^2 | \ff(r) = e_1 \right ] \lambda(\vphi)^{2(k-\ell(u,v))}  - \lambda(\vphi)^{2k} \\
& \leq \frac{1}{q^{2k}} \sum_{u,v \in T_k} \lambda(\vphi)^{2(k-\ell(u,v))}  - \lambda(\vphi)^{2k} \\
& = \frac{\lambda(\vphi)^{2k}}{q^{2k}} \sum_{u,v \in T_k} \lambda(\vphi)^{-2\ell(u,v)}  - \lambda(\vphi)^{2k} \\
& \leq \frac{\lambda(\vphi)^{2k}}{q^{2k}} \sum_{\ell=0}^k q^\ell q^{2(k-\ell)} \lambda(\vphi)^{-2\ell} - \lambda(\vphi)^{2k} = \lambda(\vphi)^{2k} \sum_{\ell=1}^k \left (q \lambda(\vphi)^2 \right )^{-\ell},
\end{align*}
where in the third equality, we have applied \eqref{eq:markovspectrum} and \eqref{eq:markovprojection} to $\EE [\ff(v)_1 | \ff(c(u,v))]$ and  $\EE[\ff(u)_1 | \ff(c(u,v))]$.
This gives a sufficient condition for \eqref{eq:flowdecay} to hold true, concluding:
\begin{claim}
The condition $q \lambda(\vphi)^2 > 1$ is sufficient for the model to have positive percolation of information.
\end{claim}
\noindent We will refer to this as the Kesten-Stigum (KS) bound. \\

We now turn to describe our lower bounds. For the Gaussian kernel, we give a lower bound which misses by a factor of $2$ from giving a matching bound to the KS bound. To describe the Gaussian kernel, fix $\beta > 0$, let $X$ be a normal random variable with law $\mathcal{N}(0,\beta)$ and suppose that $\vphi : \mathbb{S}^1 \times \mathbb{S}^1 \to \RR $ is such that 
\begin{equation} \label{eq:gausskernel}
\varphi(x, \cdot)\text{ is the density of }(x + X) \mathrm{mod} 2 \pi, 
\end{equation} where we identify $\mathbb{S}^1$ with the interval $[0,2 \pi)$. We have the following result.

\begin{theorem} \label{thm:Gaussian2d}
For the Gaussian kernel defined above, there is zero information flow whenever $q \lambda(\vphi) < 1$. 
\end{theorem}
\paragraph{Related results in the Gaussian setting.} In the discrete setting, an analogous result was obtained in \cite{mossel2001reconstruction}. In fact, our method of proof is closely related. We use the inherent symmetries of the spherical kernels to decouple the values of $g(r)$ and $g(v)$, where $v$ is some vertex which is sufficiently distant from $r$. This is same idea of Proposition 10 in\cite{mossel2001reconstruction} which uses a decomposition of the transition matrix to deduce a similar conclusion.\\
Another related result appears in \cite{mossel2013robust}. In the paper the authors consider a broadcast model on a binary tree with a Gaussian Markov random field and obtain a result in the same spirit as the one above. However, since the random field they consider is real valued, as opposed to our process which is constrained to the circle, the method of proof is quite different.\\

In the general case, we were unable to give a corresponding bound, nevertheless, using the same ideas, we are able to give some nontrivial sufficient condition for zero flow of information for some $q>1$, formulated in terms of the eigenvalues of the kernel. To our knowledge, this is the first known result in this direction. In order to formulate our result, we need some definitions.

We begin with a slightly generalized notion of a $q$-ary tree.
\begin{definition}
	Let $q > 1$, we say that $T$ is a tree of growth at most $q$ if for every $k \in \NN$, 
	$$\left|T_k\right| \leq \left\lceil q^k \right\rceil.$$  
\end{definition}
Now, recall that $\vphi(x,y) = f(\inner{x}{y})$. 
Our bound is proven under the following assumptions on the kernel.
\begin{itemize} \label{ass:kernel}
	\item $f$ is monotone.
	\item $f$ is continuous.
	\item $\lambda(\vphi) > 0$ and for every $i \geq 1$, $ |\lambda_i| \leq \lambda(\vphi)$.
\end{itemize}

\noindent We obtain the following result.
\begin{theorem} \label{thm:lowerbound}
	Let $\vphi$ satisfy the assumptions above and let $T$ be a tree of growth at most $q$. There exists a universal constant $c > 0$, such that if 
	$$
	q \leq \left(1-c\frac{\ln(\lambda(\vphi))(1-\lambda(\varphi))^2}{\ln\left(\frac{\lambda(\vphi)(1-\lambda(\vphi))}{f(1)}\right)}\right)^{-1}
	$$
	then the model has zero percolation of information.
\end{theorem}
\paragraph{Possible connection to reconstruction on random geometric graphs.}
	One might naively expect that the machinery developed in \cite{MosPer03} might be used, in conjunction with Theorems \ref{thm:Gaussian2d} and \ref{thm:lowerbound}, to deduce lower bounds for reconstructability in random geometric graphs. However, in contrast to the SBM, if $v$ is a vertex and we condition on the latent positions of all vertices in the boundary of some ball around $v$, the position of $v$ still retains some dependency with the position of vertices outside the ball. 
	
	Because of this fact, reducing a lower bound for the problem of recovering latent positions in random geometric graphs to the continuous broadcast model cannot be as straightforward as in the SBM. We leave the existence of a possible reduction as an interesting question for future research.
\section{The upper bound: Proof of Theorem \ref{thm: reconstruction bound}}

Recall the following identity in $L^2(\Spd)$,
\begin{displaymath} 
\vphi = \sum\limits_{k=0}^\infty \lambda_k \psi_k \otimes \psi_k,
\end{displaymath}
with the eigenvalues $\lambda_k$ indexed by the spherical harmonics.  Define the random matrices $M_n, \Psi_n$ by
$$(M_n)_{i,j} = \frac{1}{n}\vphi(X_i,X_j),  ~~~~~ (\Psi_n)_{i,k} = \frac{1}{\sqrt{n}}\psi_k(X_i).$$
Note that $M_n$ is an $n \times n$ matrix, while $\Psi_n$, has infinitely many columns. Furthermore, denote by $\Lambda$ the diagonal matrix $\mathrm{diag}\{\lambda_i\}_{i=0}^\infty$. Then, $\sigma$-almost surely, for every $i,j \in [n],$
\[
(M_n)_{i,j} = \frac{1}{n}\sum _{k=0}^\infty \lambda_k\psi_k(X_i)\psi_k(X_j) =(\Psi_n\Lambda\Psi_n^T)_{i,j}.
\]
For $r \in \NN$ we also denote
$$\vphi^r := \sum\limits_{k=0}^r \lambda_k \psi_k \otimes \psi_k,$$
the finite rank approximation of $\vphi$, $\Lambda^r = \mathrm{diag}\{\lambda_k\}_{k=0}^r$, and $\Psi_n^r$ the sub-matrix of $\Psi_n$ composed of its first $r$ columns. Finally, denote
$$M_n^r =\Psi_n^r\Lambda^r(\Psi_n^r)^T.$$
As before, let $A$ be an adjacency matrix drawn from $G\left(n, \frac{1}{n}\vphi(X_i,X_j)\right)$ so that $\EE \left[A|X_1,\dots,X_n\right] = M_n$. Our goal is to recover $\Psi_n^{d+1}$ from the observed $A$. 
The first step is to recover $\Psi_n^{d+1}$ from $M_n$. We begin by showing that the columns of $\Psi_n^r$ are, up to a small additive error, eigenvectors of $M_n^r$.
To this end, denote
$$E_{n,r} := (\Psi_n^r)^T\Psi_n^r - \mathrm{Id}_r,$$
$C(n,r) = \norm{E_{n,r}}^2_{op}$, and $K = \max_i{\lambda_i}$. 
\begin{lemma} \label{lem: lambda approx}
	Let $u_i$ be the $i$'th column of $\Psi_n$ and let $\eta >0$. Then
	\[
	\norm{M_n^ru_i - \lambda_iu_i}_2^2\leq K^2(\sqrt{C(n,r)}+1)C(n,r).
	\]
	Moreover, whenever $n\geq l(r) \log (2r/\eta)$, we have with probability larger than $1-\eta$,
	\[
	C(n,r) \leq \frac{4 l(r)\log(2r/\eta)}{n}.
	\]
	where $l(r)$ only depends on $r$ and on the dimension.
\end{lemma}
\begin{proof}
	Let $e_i\in\RR^r$ be the $i$'th standard unit vector so that $u_i = \Psi_n^re_i$. So,
	\[(\Psi_n^r)^Tu_i = (\Psi_n^r)^T\Psi_n^re_i =(\mathrm{Id}_r + (\Psi_n^r)^T\Psi_n^r - \mathrm{Id}_r)e_i=e_i + E_{n,r}e_i.\]
	We then have
	\begin{align*}
	M_n^ru_i &= \Psi_n^r\Lambda^r(\Psi_n^r)^Tu_i = \Psi_n^r\Lambda^re_i + \Psi_n^r\Lambda^rE_{n,r}e_i \\
	&= \lambda_i\Psi_n^re_i + \Psi_n^r\Lambda^rE_{n,r}e_i = \lambda_iu_i + \Psi_n^r\Lambda^rE_{n,r}e_i.
	\end{align*}
	To bound the error, we estimate $\norm{M_n^ru_i - \lambda_iu_i}_2^2=\norm{\Psi_n^r\Lambda^rE_{n,r}e_i}_2^2$ as
	\begin{align*}
	\inner{\Lambda^rE_{n,r}e_i}{(\Psi_n^r)^T\Psi_n^r\Lambda^rE_{n,r}e_i} &= \inner{\Lambda^rE_{n,r}e_i}{E_{n,r}\Lambda^rE_{n,r}e_i}+ \norm{\Lambda^rE_{n,r}e_i}_2^2 \\
	&\leq \left(\sqrt{C(n,r)}+1\right)\norm{\Lambda^rE_{n,r}e_i}_2^2 \leq K^2\left(\sqrt{C(n,r)}+1\right)C(n,r).
	\end{align*}
	
	It remains to bound $C(n,r)$. Let  $X_i^r = ( \psi_0(X_i),\dots, \psi_{r-1}(X_i))$ stand for the $i$'th row of $\Psi_n^r$. Then, $E_{n,r} = \frac{1}{n}\sum_{i=1}^n\left((X_i^r)^T X_i^r - \mathrm{Id}_r\right)$, is a sum of independent, centered random matrices.
	We have
	\begin{align*}
	\sigma_{n,r}^2 &:= \norm{ \EE\left( \frac{1}{n}\sum_{i=1}^n\left((X_i^r)^T X_i^r - \mathrm{Id}_r\right)\right)^2}_{op} \\
	&= \frac{1}{n} \norm{\EE \left( (X_1^r)^T X_1^r - \mathrm{Id}_r\right)^2}_{op} 
	\end{align*}
	Furthermore, the norm of the matrices can be bounded by 
	\begin{align*}
	\norm{\frac{1}{n}\left((X_1^r)^T X_1^r - \mathrm{Id}_r\right)}_{op} &=\frac{1}{n} \max(1, \norm{X_1^r}_2^2 -1) \\
	&\leq \frac{1}{n}\max\left(1, \norm{ \sum_{i=0}^r \psi_i^2}_\infty - 1\right).
	\end{align*}
	Note that the right hand side of the two last displays are of the form $\frac{1}{n} l(r)$ where $l(r)$ depends only on $r$ and $d$ (not not on $n$). Applying matrix Bernstein (\cite[Theorem 6.1]{Tropp12}) then gives
	\[
	\PP\left( \norm{E_{n,r}}_{op}\geq t \right) \leq 2r \exp\left(-\frac{n}{2l(r)}\frac{t^2}{1+t/3} \right),
	\]
	where $l(r)$ depends only on $r$ and $d$. Choose now $t_0 = \frac{4 l(r)\log(2r/\eta)}{n}$. As long as $n\geq l(r) \log (2r/\eta)$, $t_0 \leq 4$, and the above bound may be refined to
	\[
\PP\left( \norm{E_{n,r}}_{op}\geq t_0 \right) \leq 2r \exp\left(-\frac{n}{l(r)}\frac{t_0^2}{7} \right).
	\]
	With the above conditions, it may now be verified that $2r \exp\left(-\frac{n}{l(r)}\frac{t_0^2}{7} \right) \leq \eta$, and the proof is complete.
\end{proof}
We now show that as $r$ increases, the subset of eigenvectors of $M_n^r$, that encode the latent positions of $\{X_i\}_{i=1}^n$, converge to those of $M_n$. Order the eigenvalues in decreasing order $\lambda_{l_0}\geq \lambda_{l_1} \geq \lambda_{l_2} \geq\dots$ and  let $\Lambda_{>r} = \sum_{i=r}^\infty \lambda_{l_i}^2$. Note that it follows from assumption \ref{assumption 2} that $\Lambda_{>r} = O(r^{-3})$. We will denote by $\lambda_i(M_n), \lambda_i(M_n^r)$ the respective eigenvalues of $M_n$ and $M_n^r$, ordered in a decreasing way, and by $v_i(M_n),v_i(M_n^r)$ their corresponding unit eigenvectors. Suppose that $s$ is such that
\begin{equation}\label{eq:defs}
\lambda(\vphi) = \lambda_{l_{s+1}} = \dots = \lambda_{l_{s+d}}.
\end{equation}
Moreover, define
$$
V_n := \mathrm{span}(v_{l_{s+1}}(M_n),..., v_{l_{s+d}}(M_n)), ~~~ V_n^r := \mathrm{span}\left(v_{l_{s+1}}(M_n^r),..., v_{l_{s+d}}(M_n^r)\right).
$$
The next lemma shows that $V_n$ is close to $V_n^r$ whenever both $n$ and $r$ are large enough.
\begin{lemma}  \label{lem: finite rank}
	Let $\delta$ and $C$ be the constants from Assumption \ref{assumption1} and Assumption \ref{assumption 2}. For all $n,r$, let $P_{n,r}$ be the orthogonal projection onto $V_n^r$. Then, for all $\eta > 0$ there exist constants $n_0,r_0$, which may depend on $\delta$, $C$, and $\eta$, such that for all $n>n_0$ and $r>r_0$, we have with probability at least $1-\eta$ that, for all $w\in V_n$,
	$$\norm{w - P_{n,r}w}_2 \leq \frac{4C}{\eta\delta^2 r^3 }.$$
\end{lemma}
\begin{proof}
	We have
	\begin{align*}
	\EE \norm{M_n - M_n^r}^2_{F} &= \sum_{i,j} \EE (M_n -M_n^r)^2_{i,j} \\
	& = \sum_{i,j} \frac{1}{n^2} \EE \left(\sum_{k=r}^\infty \lambda_k\psi_k(X_i)\psi_k(X_j)\right)^2 \\
	& =  \EE_{x,y \sim \sigma} \left(  \sum_{k=r}^\infty \lambda_k \psi_k(x)\psi_k(y)\right)^2 \\
	& = \sum_{k=r}^\infty \lambda_k^2 = \Lambda_{>r}.
	\end{align*}
	Applying Markov's inequality gives that with probability $1-\eta$
	\begin{equation}
	\label{eq: approx Mn}
	\norm {M_n - M_n^r}^2_{F} \leq \frac{\Lambda_{>r}}{\eta}\leq\frac{C}{\eta r^3}.
	\end{equation}
	Theorem 1 in \cite{valdivia2018relative} shows that there exists $n$ large enough such that with probability larger than $1-\eta$, one has
	\[
	|\lambda_i(M_n) - \lambda_{l_{s+i}}| \leq \delta / 4,
	\]
	with $\delta$ being the constant from Assumption \ref{assumption1}. It follows that
	\begin{equation}\label{eq:lambdas}
	\lambda_{l_{s+1}}(M_n),...,\lambda_{l_{s+d}}(M_n) \in \left[\lambda(\vphi)-\frac{\delta}{4}, \lambda(\vphi)+\frac{\delta}{4}\right],
	\end{equation}
	while by \eqref{eq: approx Mn} and Weyl's Perturbation Theorem (e.g., \cite[Corollary III.2.6]{Bhatia1997}), for $r$ large enough with probability $1-\eta$,
	\begin{equation}\label{eq:otherlambdas}
	\lambda_i(M_n^r) \notin \left[\lambda(\vphi)-\frac{3\delta}{4}, \lambda(\vphi)+\frac{3\delta}{4}\right], \text{ for } i \neq l_{s+1},\dots l_{s+d}.
	\end{equation}
	Combining \eqref{eq: approx Mn}, \eqref{eq:lambdas} and \eqref{eq:otherlambdas} it follows from the classical Davis-Kahan theorem (see e.g. \cite[Section VII.3]{Bhatia1997}) that with probability at-least $1-2\eta$, for every $w \in V_n$,
	$$\norm{w - P_{n,r}w}_2^2 \leq \frac{4C}{\eta\delta^2 r^3}.$$
\end{proof}
Denote
$$
G_n := \frac{1}{d} \sum_{k=1}^d v_{l_{s+k}}(M_n) v_{l_{s+k}}(M_n)^T, ~~~ (G_n')_{i,j} = \frac{1}{n}\langle X_i, X_j \rangle.  
$$
A combination of the last two lemmas produces the following:
\begin{theorem}
	\label{thm:eigenvectors_converge}
	One has
	$$
	\| G_n - G_n' \|_F \to 0
	$$
	in probability, as $n \to \infty$.
\end{theorem}
\begin{proof}
	Denote
	$$
	G_n^r:=\frac{1}{d}\sum_{k=1}^d v_{l_{s+k}}(M_n^r)(v_{l_{s+k}}(M_n^r))^T.
	$$
	Then
	$$
	\| G_n - G_n' \|_F \leq \|G_n^r - G_n'\|_F + \|G_n-G_n^r\|_F
	$$
	We will show that the two terms on the right hand side converge to zero. Let $r(n)$ be a function converging to infinity slowly enough so that $C(n,r)\to 0$, for the constant $C(n,r)$ defined in Lemma \ref{lem: lambda approx}. Taking $\eta = \eta(n)$ to converge to zero slowly enough and applying Lemma \ref{lem: lambda approx}, gives for all $1 \leq i \leq d$,
	$$
	\|(M_n^r - \lambda(\vphi)\mathrm{Id}_n) u_i \|_2^2 \leq \eps_{n}
	$$
	with $u_i$ the $i$'th column of $\Psi_n^{r}$ and where $\eps_n \to 0$ as $n\to \infty$. Now, if we write
	$$
	u_i = \sum_{j=0}^\infty \alpha_{i,j} v_j(M_n^r),
	$$
	the last inequality becomes
	$$
	\sum_j |\lambda_j(M_n^r) -\lambda(\vphi)|^2 \alpha_{i,j}^2 = \sum_j |(M_n^r - \lambda(\vphi)\mathrm{Id}_n) v_j(M_n^r) |^2 \alpha_{i,j}^2 \leq \eps_n, ~~ \forall 1 \leq i \leq d.
	$$
	Recall that, by \eqref{eq:otherlambdas}, for $j \notin \{l_{s+1},..,l_{s+d}\}$, we have $\frac{3\delta}{4}\leq |\lambda_j(M_n^r) -\lambda(\vphi)|$. Hence.
	\begin{equation}\label{eq:alphas}
	\sum_{j \notin \{l_{s+1},..,l_{s+d}\}} \alpha_{i,j}^2 \leq \frac{16 \eps_n}{9\delta^2} \to 0,
	\end{equation}
	and thus
	$$
	\left\|u_i - \sum_{k=1}^d \alpha_{i, l_{s+k}} v_{l_{s+k}}(M_n^r)\right\|_2^2\to 0,  ~~ \forall 1 \leq i \leq d.
	$$
	Define a $d\times d$-matrix $B$ by $B_{i,j}=\alpha_{i, l_{s+j}}$. Then we can rewrite the above as
	$$
	\left\|(u_1,\dots,u_d) - (v_{l_{s+1}}(M_n^r),\dots, v_{l_{s+d}}(M_n^r))\cdot B\right\|_F^2 \to 0. 
	$$
	Now, since for two $n\times d$ matrices $R,S$ we have $\|RR^T - SS^T\|_F\leq (\|R\|_{op} +\|S\|_{op})\|R-S\|_F$. It follows that
	\begin{equation}\label{eq:interm1}
	\|G_n' - (v_{l_{s+1}}(M_n^r), \dots, v_{l_{s+d}}(M_n^r))BB^T(v_{l_{s+1}}(M_n^r), \dots, v_{l_{s+d}}(M_n^r))^T\|_F \to 0. 	
	\end{equation}
	Observe that 
	$$
	B_{i,j} = \langle u_i, u_j \rangle - \sum_{k \notin \{l_{s+1},..,l_{s+d}\}} \alpha_{i,k} \alpha_{j,k},
	$$
	implying that
	$$
	|(B B^T)_{i,j} - (E_{i,j} + \mathrm{Id}_d)| \leq \sqrt{ \sum_{k \notin \{l_{s+1},..,l_{s+d}\}} \alpha_{i,k}^2 \sum_{k \notin \{l_{s+1},..,l_{s+d}\}} \alpha_{j,k}^2 } \stackrel{ \eqref{eq:alphas} }{\longrightarrow} 0.
	$$
	where $E = E_{n,r}$. Consequently we have
	$$
	\|B B^T - \mathrm{Id}_d \|_{op} \leq \sqrt{ \sum_{k \notin \{l_{s+1},..,l_{s+d}\}} \alpha_{i,k}^2 \sum_{k \notin \{l_{s+1},..,l_{s+d}\}} \alpha_{j,k}^2 } + \sqrt{C(n,r)} \to 0,
	$$
	which implies that
	\begin{align*}
	\|v_{l_{s+1}}(M_n^r), \dots, v_{l_{s+d}}(M_n^r))(BB^T-\mathrm{Id}_d)(v_{l_{s+1}}(M_n^r), \dots, v_{l_{s+d}}(M_n^r))^T\|_F^2\to 0.
	\end{align*}
	Combining with \eqref{eq:interm1} finally yields
	$$
	\|G_n' - G_n^r\|_F \to 0.
	$$
	in probability, as $n\to\infty$.
	\\
	If $P$ is the orthogonal projection onto $V^r_n=\mathrm{span}(v_{l_{s+1}}(M_n^r),\dots,v_{l_{s+d}}(M_n^r))$, and $Q$ is the orthogonal projection onto $\mathrm{span}(v_{l_{s+1}}(M_n),\dots,v_{l_{s+d}}(M_n))$, then Lemma \ref{lem: finite rank} shows that for all $\eta > 0$, with probability at least $1-\eta$, as $n\to \infty$ (and $r=n^{\frac{1}{2d}}\to\infty$), we have for every unit vector $v$
	\begin{equation}\label{eq:PnotP}
	|(P - \mathrm{Id}_n) Q v | \leq \eps_n
	\end{equation}
	with some $\eps_n \to 0$. By symmetry, we also have for every unit vector $v$ that 
	$$
	|(Q - \mathrm{Id}_n) P v | \leq \eps_n 
	$$
	(this uses that fact that both $P$ and $Q$ are projections into subspaces of the same dimension). The last two inequalities easily yield that $\|P - Q\|_{op} \leq \eps_n$. Since this is true for every $\eta > 0$, it follows that
	$$
	\|G_n-G_n^r\|_F\to 0 ,
	$$
	in probability, as $n\to\infty$. 
\end{proof}
Now, after establishing that $M_n$ is close to $\Psi_n^{d+1}$, the second step is to recover $M_n$ (and therefore $\Psi_n^{d+1}$), from the observed $A$. For the proof we will need the following instance of the Davis-Kahan theorem.
\begin{theorem} [{\cite[Theorem 2]{YuWaSa14}}]
	\label{thm:Davis-Kahan}
	Let $X,Y$ be symmetric matrices with eigenvalues $\lambda_0\geq\dots \geq\lambda_p$ resp. $\hat\lambda_0\geq\dots\geq\hat{\lambda}_p$ with corresponding orthonormal eigenvectors 			$v_0,\dots,v_p$ resp. $\hat v_0,\dots, \hat v_p$. Let $V=(v_{s+1},\dots, v_{s+d})$ and $\hat V = (\hat v_{s+1},\dots, \hat v_{s+d})$. Then there exists an orthogonal $ d \times d$ 	matrix $R$ such that
	\[
	\| \hat V R - V \|_{F} \leq \frac{2^{3/2} \min(d^{1/2}\| Y-X\|_{op}, \|Y-X\|_F)}{\min(\lambda_{s}-\lambda_{s+1}, \lambda_{s+d}-\lambda_{s+d+1})}.
	\]
\end{theorem}
Our main tool to pass from the expectation of the adjacency matrix to the matrix itself is the following result regarding concentration of random matrices, which follows from \cite[Theorem 1.1]{LeLevVer17} (see also \cite[Theorem 5.1]{LeLevVer18}). 
\begin{theorem}
	\label{thm:perturbationbound}
	Let $A$ be the adjacency matrix of a random graph drawn from $G\left(n, \frac{1}{n} \vphi(X_i, X_j)\right)$.
	 For every vertex of degree larger than $\max_{ij}\varphi(X_i,X_j)$, remove all incident edges and let $A'$ denote the new adjacency matrix.
	 Then,
	\[
	\|A'-M_n\| \leq C \sqrt{\|\vphi\|_\infty}.
	\]
\end{theorem}

We can now prove the main reconstruction theorem.
\begin{proof}[Proof of Theorem \ref{thm: reconstruction bound}]
	
	Let $A$ be the adjacency matrix of a random graph drawn from the model $G\left(n, \frac{1}{n} \vphi(X_i, X_j)\right)$ and let $A'$ be the adjacency matrix defined in Theorem \ref{thm:perturbationbound}. 
	
	Denote by $\lambda_0'\geq\lambda_1'\geq\dots$ the eigenvalues of $A'$ and by $v_0', v_1',\dots$ the corresponding orthonormal eigenvectors. Let $Y = (v_{l_{s+1}}',\dots,v_{l_{s+d}}')$. By Theorem \ref{thm:Davis-Kahan} there exists an $R\in \mathcal{O}(d)$ such that 
	\begin{align*}
	\norm{ \left(v_{l_{s+1}}(M_n),...,v_{l_{s+d}}(M_n)\right) - YR }_{F} &\leq \frac{2^{3/2} d^{1/2}\| M_n - A'\|_{op}}{\min_{i\neq 1,\dots,d}|\lambda_{i}-\lambda(\vphi)|}.
	\end{align*}
	Hence by Theorem \ref{thm:perturbationbound} we have
	\[
	\norm{ \left(v_{l_{s+1}}(M_n),...,v_{l_{s+d}}(M_n)\right) - YR }_{F} \leq C \sqrt{ d} \cdot \frac{\sqrt{\|\varphi\|_\infty}}{ \min_{i\neq 1,\dots,d}|\lambda_{i}-\lambda(\vphi)|},
	\]
	with probability $1-n^{-1}$. It follows that
	$$
	\norm{ G_n - Y Y^T  }_{F} \leq  C \sqrt{ d} \cdot \frac{\sqrt{\|\varphi\|_\infty}}{ \min_{i\neq 1,\dots,d}|\lambda_{i}-\lambda(\vphi)|}
	$$
	Combining this with Theorem \ref{thm:eigenvectors_converge} yields 
	\begin{align*}
	\norm{  G_n' - Y Y^T \ }_{F} \leq  C \sqrt{d} \cdot \frac{\sqrt{\|\varphi\|_\infty}}{\min_{i\neq 1,\dots,d}|\lambda_{i}-\lambda(\vphi)|},
	\end{align*}
	So,	
	$$
	\frac{1}{n^2} \sum_{i,j} \left |X_i \cdot X_j - (n Y Y^T)_{i,j} \right |^2 \leq  C d  \frac{\|\varphi\|_\infty}{\min_{i\neq 1,\dots,d}|\lambda_{i}-\lambda(\vphi)|^2} 
	$$
	which gives the desired reconstruction bound.	

\end{proof}

\section{Lower bounds} \label{sec:lower}

Our approach to proving lower bounds will be to exploit some symmetries which are inherent to well behaved kernel functions. We thus make the following definition:
\begin{definition}[DPS property]
	Let $\mu$ be a probability measure on $\Spd$, and let $w \in \Spd$. We say that $\mu$ has the Diminishing Post-translation Symmetry (DPS) around $w$ property with constant $\eps$, if there exists a decomposition,
	\[\mu = (1-\eps)\mu_w + \eps \mu_w^s.\] 
	Here $\mu_w^s$ is a probability measure, invariant to reflections with respect to $w^\perp$. 
	In other words, if $R = \mathrm{Id}_d - 2ww^T$, $R_*\mu_w^s = \mu_w^s$. 
	For such a measure we denote $\mu \in \mathrm{DPS}_w(\eps)$.\\
	
	If instead $\mu$ is a measure on $(\Spd)^{|T_k|}$, we say that $\mu \in \mathrm{DPS}^k_w(\eps)$ if a similar decomposition exists
	but now the reflections should be applied to each coordinate separately.
\end{definition}
We now explain the connection between the DPS property and percolation of information in trees. For this, let us recall the random function $g:T \to \Spd$, introduced in Section \ref{sec:perco}, which assigns to the root, $r$, a uniformly random value and for any other $u \in T$, the label $g(u)$ is distributed according to $\vphi(g(\mathrm{parent}(u)),\cdot)d\sigma$.
\begin{lemma} \label{lem:generallower}
	Suppose that there exist a sequence $(p_k)_k$ with $\lim_{k\to \infty} p_k= 1$ such that for every $w,x_0 \in \Spd$ and every $k > 0$,
	\[\mathrm{Law}((g(v))_{v\in T_k}|g(r) = x_0) \in \mathrm{DPS}^k_w(p_k).\]
	Then there is zero percolation of information along $T$.
\end{lemma}
\begin{proof}
	 Denote $X = g(r)$ and $Y = g(v)_{v\in T_k}$ and let $\rho_{X|Y}$ be the density of $X|Y$ with respect to $\sigma$.
	 Our aim is to show that $\EE_Y\left[\mathrm{TV}(X|Y,\sigma)\right] = o(1)$. We first claim that it is enough to show that for all $x,x' \in \Spd$ and all $\delta > 0$ one has,
	\begin{equation} \label{eq:o(1)}
	\PP \left . \left ( \frac{\rho_{X|Y}(x)  }{ \rho_{X|Y}(x') } - 1 \leq \delta \right | X = x  \right ) = 1-o(1).
	\end{equation}
	Indeed, let $H = \left  \{\frac{\rho_{X|Y}(x)  }{ \rho_{X|Y}(x') } - 1 \leq \delta \right \}$. Note that, by Bayes' rule,
	$$
	\frac{\rho_{X|Y}(x)  }{ \rho_{X|Y}(x') } = \frac{\rho_{Y|X=x}(Y)  }{ \rho_{Y|X=x'}(Y) }.
	$$	
	Let $x \in \Spd$ be some fixed point and let $x'$ be uniformly distributed is $\Spd$ and independent from $Y$. We will use $\EE_{x'}$ to denote integration with respect to $x'$, which has density $\rho_X$. Similarly, $\EE_{Y}$ stands for integration with respect to the density $\rho_Y$ of $Y$ and $\mathbb{E}_{x', Y}$ is with respect to the product $\rho_X\cdot \rho_Y$. The symbol $\PP$ will always mean probability over $x'$ and $Y$.
	
	As a consequence of the assumption in \eqref{eq:o(1)}, as long as $k$ is large enough, we have,
	$$\PP\left(H|X=x\right) \geq  1-\delta.$$
	For any Borel measurable set $A$, we denote the projection:
	\begin{align*}
	\Pi_YA &= \{y \in (\Spd)^{|T_k|}: (\tilde{x}, y) \in A \text{ for some } \tilde{x} \in \Spd\}.
	\end{align*}
	Define now the set 
	$$H' := \{(\tilde{x}, y)| (\tilde{x},y)\in H \text{ and } \mathbb{E}_{x'}[\mathbf{1}_H(\cdot, y)]\geq 1 - 4\delta\}.$$
	 In particular, for every $y \in \Pi_YH'$,
	$$\mathbb{E}_{x'}[\mathbf{1}_H(\cdot, y)]\geq 1 - 4\delta,$$
	and by Fubini's theorem,  $\PP\left(H'|X = x\right) \geq 1 - 4\delta$.
	
	Consider the random variables $\alpha:= \alpha(Y) = \frac{\rho_{Y|X=x}(Y)}{ \rho_{Y}(Y) }$ and $\beta:=\beta(x',Y) = \frac{\rho_{Y|X=x'}(Y)}{\rho_Y(Y)}$. 
	By definition of $H$ , we have 
	\begin{equation} \label{eq:betalarge}
	(1-\delta)\mathbf{1}_H\alpha\leq \mathbf{1}_H\beta.
	\end{equation}
	Moreover, for almost
	 every $Y$,
	$$\EE_{x'}\left[\beta\right] = \frac{1}{\rho_Y(Y)}\int\rho_{Y|X=x'}(Y)\rho_X(x') = \frac{1}{\rho_Y(Y)}\rho_Y(Y) = 1.$$
	 So,
	\[
	(1-\delta)(1-4\delta)\EE_Y\left[\alpha\mathbf{1}_{ \Pi_Y{H'}}\right]\leq(1-\delta)\EE_{Y}\left[\alpha\EE_{x'}\left[\mathbf{1}_{H'}\right]\right]	\leq \EE_{x',Y}\left[\beta\mathbf{1}_{H'}\right]\leq\EE_{x',Y}\left[\beta\right]=1.
	\]
	Observe that $\EE_Y\left[\alpha\mathbf{1}_{ \Pi_Y{H'}}\right] =  \PP\left(\Pi_Y{H'}|X=x\right) =  1- o(1)$, where the second equality is a consequence of Fubini's theorem. Hence, let us write $h_1(\delta)$, so that
	$$1 - h_1(\delta) = (1-\delta)(1-4\delta)\EE_Y\left[\alpha\mathbf{1}_{ \Pi_Y{H'}}\right] \leq (1-\delta)\EE_{x', Y}\left[\alpha \mathbf{1}_{H'}\right].$$
	Markov's inequality then implies,
	\begin{equation}\label{eq:betasmall}
	\PP\left(\beta\mathbf{1}_{H'}\geq (1-\delta)\alpha\mathbf{1}_{H'}+\sqrt{h_1(\delta)}\right)\leq \frac{\EE_{x', Y}\left[\beta\mathbf{1}_{H'}\right]-(1-\delta)\EE_{x', Y}\left[\alpha\mathbf{1}_{H'}\right]}{\sqrt{h_1(\delta)}}\leq \sqrt{h_1(\delta)}.
	\end{equation}
	Now, we integrate over $x'$ to obtain,
	$$(1-\delta)(1-4\delta)\alpha\mathbf{1}_{\Pi_Y{H'}}\leq(1-\delta)\alpha\EE_{x'}\left[\mathbf{1}_{H'}\right] \leq \EE_{x'}\left[\beta\mathbf{1}_{H'}\right] \leq 1,$$
	which implies
	\begin{equation} \label{eq:alphalarge}
	\alpha\mathbf{1}_{H'} \leq \frac{1}{(1-\delta)(1-4\delta)}.
	\end{equation}
	Keeping in mind the previous displays, we may choose $h_2(\delta)$, which satisfies $\lim\limits_{\delta \to 0} h_2(\delta) = 0$, $\alpha\mathbf{1}_{H'} \leq 1 + h_2(\delta)$ and $\EE_{x', Y}\left[\alpha \mathbf{1}_{H'}\right] \geq 1 - h_2(\delta)$.
	
	So, an application of the the reverse Markov's inequality for bounded and positive random variables shows,
	\begin{align} \label{eq:alphasmall}
	\PP\left(\alpha\mathbf{1}_{H'}\geq 1-\sqrt{h_2(\delta)}\right) \geq \frac{\EE_{x', Y}\left[\alpha\mathbf{1}_{H'}\right] - (1-\sqrt{h_2(\delta)})}{1+h_2(\delta) - (1-\sqrt{h_2(\delta)})} 
	&\geq \frac{1-h_2(\delta) - (1-\sqrt{h_2(\delta)})}{1+h_2(\delta) - (1-\sqrt{h_2(\delta)})}\nonumber \\
	&= \frac{\sqrt{h_2(\delta)} - h_2(\delta)}{\sqrt{h_2(\delta)} + h_2(\delta)}. 
	\end{align}
	Note that as $\delta \to 0$ the RHS goes to $1$. Thus, by combining the above displays, there exists a function $h$, which satisfies $\lim\limits_{\delta \to 0}h(\delta) = 0$ and some $H'' \subset H'$, with $\PP\left(H'\right) \geq 1 -h(\delta)$, such that, by \eqref{eq:alphalarge} and \eqref{eq:alphasmall},
	\[
	\mathbf{1}_{H''}|\alpha -1|\leq h(\delta),
	\]
	which implies, together with \eqref{eq:betalarge} and \eqref{eq:betasmall}
	\[
	\mathbf{1}_{H''}|\alpha -\beta|\leq h(\delta).
	\]
	This then gives 
	\[
	\mathbf{1}_{H''}|1 -\beta|\leq 2h(\delta).
	\]
	 We can thus conclude,
	 \begin{align*}
	 \EE_Y \mathrm{TV} \left ( X|Y, X \right ) &= \EE_{Y,x'} \left [ |\beta - 1| \right ] = 2\EE_{Y,x'} \left [(1-\beta)\mathbf{1}_{\beta \leq 1} \right ] \\
	 &\leq 2 \EE_{Y,x'} \left [(1-\beta)\mathbf{1}_{\beta \leq 1}\mathbf{1}_{H''}\right ] +1-\PP\left(H''\right)\\
	 &= 2 \EE_{Y,x'} \left [|1-\beta|\mathbf{1}_{H''}\right ] +h(\delta)\\
	 &\leq 5h(\delta).\\
	 \end{align*}	
	 Take now $\delta \to 0$ to get $\EE_Y \mathrm{TV} \left ( X|Y, X \right ) \to 0$.\\

	Thus, we may assume towards a contradiction that there exist $x, x' \in \Spd$ and a set $F\subset(\Spd)^{|T_k|}$, such that
	\begin{equation}\label{eq:Fbig}
	\PP\left(Y \in F | X = x \right) \geq \delta,
	\end{equation}	
	and under $\{Y \in F \}$,
	\begin{equation}\label{eq:badF}
	\frac{\rho_{X|Y}(x)}{\rho_{X|Y}(x')} \geq 1+ \delta,
	\end{equation}
	for some constant $\delta > 0$.\\
	
	Let $w\in \Spd$ be such that the reflection $R:=\mathrm{Id}_d - 2ww^T$ satisfies $Rx = x'$. Under our assumption, there exists an event $A_k$, which satisfies
	\begin{equation*}
	\PP(A_k|X=x) = 1 - o(1)
	\end{equation*}
	and such that $Y|(X=x,A_k)$ is $R$-invariant. By \eqref{eq:Fbig}, we also have
	$$\PP(A_k|X=x,Y \in F) = 1 - o(1),$$
	and thus there exists $y \in F$ such that
	\begin{equation} \label{eq:symmetricY}
	\PP(A_k | X=x,Y = y ) = 1-o(1).
	\end{equation} 
	By continuity, we can make sense of conditioning on the zero probability event $E := \bigl \{X = x, ~  Y \in \{y, R y\} \bigr \}$, in the sense of regular probabilities. Note that we have by symmetry and since $y \in F$,
	\begin{equation}\label{eq:tocontradict}
	Y = y ~~ \Rightarrow ~~ 1+\delta \leq  \frac{\rho_{X|Y}(x)}{\rho_{X|Y}(x')} = \frac{\PP( Y = y | E )}{ \PP( Y = R y | E )}.
	\end{equation}
	On the other hand, we have by definition of $A_k$,
	$$
	\PP( Y = R y | E, A_k ) = \PP( Y = y | E, A_k ),
	$$
	which implies that 
	\begin{align*}
	\PP( Y = R y | E ) & \geq \PP \bigl( \{Y = R y \} \cap A_k | E  \bigr) \\
	& = \PP \bigl( \{Y = y \} \cap A_k | E  \bigr) \\
	& \geq \PP( Y = y | E ) (1-o(1))
	\end{align*}
	which contradicts \eqref{eq:tocontradict}. The proof is complete.
\end{proof}
\subsection{The Gaussian case}
Our aim is to show that certain classes of kernel functions satisfy the DPS condition.
We begin by considering the case where the kernel $\varphi$ is Gaussian, as in \eqref{eq:gausskernel}. In this case, the function $g$ may be defined as follows. Let $T$ be a $q$-ary tree. To each edge $e \in E(T)$ we associate a Brownian motion $(B_e(t))_{t \in (0,1)}$ of rate $\beta$ such that for every node $v \in V$ we have
$$
\ff(v) = \sum_{e \in P(v)} B_e(1) ~ \mathrm{mod} 2 \pi
$$
where $P(v)$ denotes the shortest path from the root to $v$. 

For every node $v \in T_k$ let us now consider the Brownian motion $(B_v(t))_{t=0}^k$ defined by concatenating the Brownian motions $B_e$ along the edges $e \in P(v)$. Define by $E_v$ the event that the image of $B_v \mathrm{mod} \pi$ contains the entire interval $[0, \pi)$, and define $E_k = \bigcap_{v \in T_k} E_v$. Our lower bound relies on the following observation.
\begin{claim} \label{claim:uniform}
	Fix $v\in T_k$ and set $p_k := \PP(E_k)$. Then, for every $\theta ,x_0 \in \mathbb{S}^1$, 
	$\mathrm{Law}(g(v)|g(r)=x_0)\in \mathrm{DPS}^k_\theta(p_k).$
\end{claim}
\begin{proof}
	Fix $\theta \in [0, \pi)$. Given the event $E_v$, we have almost surely that there exists a time $t_v \leq k$ such that $B_v(t_v) \in \{\theta- \pi, \theta \}$. By symmetry and by the Markov property of Brownian motion, we have that the distribution of $\ff(v)$ conditioned on the event $\{B_v(t_v) = \theta, ~~ \forall v\}$ is symmetric around $\theta$. Thus, by considering 
	$(B_v(t))_{v\in T_k}$, under the event $\{t_v \leq k|v \in T_k \}$ we get that for any $x_0$, $\mathrm{Law}((g(v))_{v\in T_k}|g(r)=x_0)$ is symmetric around $\theta$. 
	So,
	\[\mathrm{Law}((g(v))_{v\in T_k}|g(r)=x_0)\in \mathrm{DPS}^k_\theta(p_k).\]
\end{proof}
We will also need the following bound, shown for example in \cite{ES16}.
\begin{lemma} \label{lem:brownianimage}
	Let $B(t)$ be a Brownian motion of rate $\beta$ on the unit circle. Then,
	\[
	\PP( \mathrm{Image} (B(s) \mathrm{mod} \pi )_{0 \leq s \leq t}  = [0, \pi)  ) \geq 1- C e^{- t \beta /2}.
	\]
\end{lemma}
We are now in a position to prove Theorem \ref{thm:Gaussian2d}.
\begin{proof}[Proof of Theorem \ref{thm:Gaussian2d}]
	Lemma \ref{lem:brownianimage} immediately implies that for all $v \in T_k$,
	\[
	\PP(E_v) \geq 1 - C e^{-k \beta / 2}.
	\]
	On the other hand, a calculation gives $\lambda(\vphi) = \EE[\cos(B_1)] = e^{-\beta/2}$. Thus, applying a union bound implies that for some constant $C > 0$, $\PP(E_k)\geq 1-Cq^k\lambda(\vphi)^k$.
	Hence, by Claim \ref{claim:uniform},
	\[\mathrm{Law}((g(v))_{v\in T_k}|g(r)=x_0) \in \mathrm{DPS}_w(1-Cq^k\lambda(\vphi)^k).\]
	The result is now a direct consequence of Lemma \ref{lem:generallower}.
\end{proof}
In the next section we generalize the above ideas and obtain a corresponding bound which holds for distributions other than the Gaussian one. 
\subsection{The general case}
\subsubsection{On symmetric functions}
We begin this section with simple criterion to determine whether a measure belongs to some $\mathrm{DPS}$ class. In the sequel, for $w \in \Spd$, we denote,
\[
H_w^+ =\{x \in \Spd| \inner{x}{w} \geq 0\}\text{ and }H_w^- =\{x \in \Spd| \inner{x}{w} < 0\}.
\]
\begin{lemma} \label{lem: montone ref}
	Let $f: [-1,1] \to \RR^+$ satisfy the assumptions of Theorem \ref{thm:lowerbound} and let $y \in \Spd$. If $\mu = f(\inner{\cdot}{y})d\sigma$,
	then for any $w \in \Spd$, $\mu \in \mathrm{DPS}_w(2\cdot p_w)$, where $p_w = \min\left(\mu(H_w^+),\mu(H_w^-)\right)$.
\end{lemma}
\begin{proof}
	Without loss of generality let us assume that $y \in H_w^+$. Monotonicity of $f$ implies $\mu(H_w^+)\geq\mu(H_w^-)$. Now, if $R = \mathrm{Id}_d - 2ww^T$ is the reflection matrix with respect to $w^\perp$, then we have for any $x \in H^-_w$,
	\[
	f(\inner{x}{y}) \leq f(\inner{Rx}{y}).
	\]
	This follows since $\inner{x}{y} \leq \inner{Rx}{y}$.\\
	
	Let us now define the measure $\tilde{\mu}_w^s$ such that
	$$ 
	\frac{d\tilde{\mu}_w^s}{d\sigma}(x) = \begin{cases}
	f(\inner{x}{y}) & \text{ if }x \in H^-_w\\
	f(\inner{Rx}{y})& \text{ if }x \in H^+_w.
	\end{cases}
	$$
	$\tilde{\mu}_w^s$ is clearly $R$-invariant and the above observation shows that $\tilde{\mu}_w^s(\Spd) \leq 1$. We can thus define $\tilde{\mu}_w = \mu - \tilde{\mu}_w^s$.\\	
	To obtain a decomposition, define $\mu_w^s = \frac{\tilde{\mu}^s_w}{\tilde{\mu}_w^s(\Spd)}$ and $\mu_w = \frac{\tilde{\mu}_w}{\tilde{\mu}_w(\Spd)}$, for which
	$$\mu = (1-\tilde{\mu}^s_w(\Spd))\mu_w + \tilde{\mu}_w^s(\Spd)\mu_w^s.$$
	The proof is concluded by noting $\tilde{\mu}_w^s(\Spd) = 2\mu(H_w^-)$. 
	
\end{proof}
Our main object of interest will be the measure $\mu_\vphi(y)$ which, for a fixed $y$, is defined by
\begin{equation} \label{eq:muphi}
\mu_\vphi(y) := \vphi (x,y)d\sigma(x) = f (\inner{x}{y})d\sigma(x).
\end{equation}
Let us now denote,
\[
\beta_d(t) := \frac{\Gamma(\frac{d}{2})}{\Gamma(\frac{d-1}{2})}(1-t^2)^{(d-3)/2},
\]
which is the marginal of the uniform distribution on the sphere.
We now show that the spectral gap of $\vphi$ may determine the $\mathrm{DPS}$ properties of $\mu_\vphi(y)$.
\begin{lemma} \label{lem:eigendps}
	Let $w \in \Spd$ and suppose that $f$ is monotone. If $|\inner{w}{y}|\leq \frac{1-\lambda(\vphi)}{16\sqrt{d}}$, then
	$$\mu_\vphi(y) \in \mathrm{DPS}_w\left(\frac{1-\lambda(\vphi)}{35}\right).$$
\end{lemma}
\begin{proof}
	Assume W.L.O.G. that $\inner{y}{w} > 0$. By Lemma \ref{lem: montone ref} it will be enough to bound $\int\limits_{H_w^-}\mu_{\vphi}(y)$ from below. Let $X \sim \mu_{\vphi}(y)$ and define $Z = \inner{X}{y}$. We have $\EE\left[Z\right] = \lambda(\vphi)$ and by Markov's inequality,
	$$\PP\left(Z \leq \frac{1+\lambda(\varphi)}{2}\right) = \PP\left(Z + 1 \leq \frac{1+\lambda(\varphi)}{2} + 1\right) \geq1 - \frac{2(\lambda(\varphi) +1)}{3 + \lambda(\varphi)}\geq \frac{1-\lambda(\vphi)}{4}.$$
	For $t \in [-1,1]$, set $S_t = \{x \in \Spd| \inner{x}{y} = t\}$ and let $\mathcal{H}^{d-2}$ stand for $d-2$-dimensional Hausdorff measure.
	To bound $\int\limits_{H_w^-}\mu_{\vphi}(y)$ we would first like to estimate
	$\frac{\mathcal{H}^{d-2}(S_t\cap H_w^-)}{\mathcal{H}^{d-2}(S_t)}$.\\
	
	We know that $0 \leq \inner{w}{y} \leq \frac{1-\lambda(\vphi)}{16\sqrt{d}}$. Denote $t_y := \inner{w}{y}$ and fix $t \leq  t_0 := \frac{1+\lambda(\vphi)}{2}$. With no loss of generality, let us write $w = e_1$ and $y = t_ye_1 + \sqrt{1-t_y^2}e_2$. Define now $z = -\sqrt{1-t_y^2}e_1+t_ye_2$.
	We claim that
	\begin{equation}\label{eq:2dpicture}
	\left\{v\in S_t\Big| \frac{\inner{v}{z}}{\sqrt{1-t^2}} \geq \frac{1}{2\sqrt{d}}\right\}\subseteq S_t\cap H_w^-.
	\end{equation}
	 If $v\in S_t$, its projection onto the plane $\mathrm{span}(y,w)$, can be written as $t\cdot y + \sqrt{1-t^2}c\cdot z$, for some $c\in[-1,1]$. So,
	$$\inner{v}{w} = t\cdot t_y - \sqrt{1-t^2}\sqrt{1-t_y^2}c.$$
	Now, whenever 
	$$c > \frac{t\cdot t_y}{\sqrt{1-t^2}\sqrt{1-t_y^2}},$$
	we get $\inner{v}{w} < 0$. Also,
	$$ \frac{t\cdot t_y}{\sqrt{1-t^2}\sqrt{1-t_y^2}}\leq \frac{1}{\sqrt{3}}\frac{t_y}{\sqrt{1-t_y^2}}\leq \frac{1}{2\sqrt{d}},$$
	where we have used $ t \leq \frac{1}{2}$ for the first inequality and $t_y \leq \frac{1}{2\sqrt{d}}$ in the second inequality.
	By combining the above displays with $\frac{\inner{v}{z}}{\sqrt{1-t^2}} = c$, \eqref{eq:2dpicture} is established.\\
	 Thus, by taking the marginal of $S_t$ in the direction of $-z$, we see
	\begin{align*}
	\frac{\mathcal{H}^{d-2}(S_t\cap H_w^-)}{\mathcal{H}^{d-2}(S_t)} &\geq \int \limits_{-1}^{-\frac{1}{2\sqrt{d}}}\beta_{d-1}(s)ds\geq \int \limits_{-\frac{1}{\sqrt{d}}}^{-\frac{1}{2\sqrt{d}}}\beta_{d-1}(s)ds\geq \frac{1}{2\sqrt{d}}\beta_{d-1}\left(\frac{1}{\sqrt{d}}\right)\\
	&\geq \frac{1}{2\sqrt{d}}\frac{\Gamma(\frac{d-1}{2})}{\Gamma(\frac{d-2}{2})}\left(1-\frac{1}{d}\right)^{(d-4)/2}\geq \frac{1}{10\sqrt{e}},
	\end{align*}
	where we used $\frac{\Gamma(\frac{d-1}{2})}{\Gamma(\frac{d-2}{2})} \geq \frac{\sqrt{d}}{5},$ valid for any $d \geq 3$. We use the above estimates with Fubini's theorem to obtain: 
	\begin{align*}
	\PP\left(X \in H_w^-\right) &= \int\limits_{-1}^1f(t)\mathcal{H}^{d-2}(S_t\cap H_w^-)dt \geq \int\limits_{-1}^{t_0}f(t)\mathcal{H}^{d-2}(S_t\cap H_w^-)dt\\
	&\geq\frac{1}{10\sqrt{e}}\int\limits_{-1}^{t_0}f(t)\mathcal{H}^{d-2}(S_t)dt = \frac{1}{10\sqrt{e}}\PP\left(Z \leq \frac{1+\lambda(\varphi)}{2}\right)\geq\frac{1-\lambda(\vphi)}{70}.
	\end{align*}
\end{proof}
\subsubsection{Mixing}
Recall the random function $g:T \to \Spd$, introduced in Section \ref{sec:perco}, which assignes to the root, $r$, a uniformly random value and for any other $u \in T$, the label $g(u)$ is distributed according to $\vphi(g(\mathrm{parent}(u)),\cdot)d\sigma=: \mu_\vphi(\mathrm{parent}(u))$.\\
Suppose that $v \in T_k$ and let $\{v_i\}_{i=0}^k$ denote the simple path from $r$ to $v$ in $T$. Fix $x_0 \in \Spd,$ for $i = 0,\dots,k$, we now regard,
$$
X_i :=g(v_i)|g(r) = x_0,
$$ 
as a random walk on $\Spd$. Observe that given $X_{i-1}$, $X_i\sim\mu_\vphi(X_{i-1}).$
The following lemma shows that this random walk is rapidly-mixing.
\begin{lemma} \label{lem:mixing}
	For $w \in \Spd$, let
	$$S(w) = \left\{u \in \Spd: |\inner{u}{w}| \leq \frac{1-\lambda(\varphi)}{16\sqrt{d}}\right\},$$
	and set $k_0 = \frac{\ln\left(\frac{\lambda(\vphi)(1-\lambda(\vphi))}{32 f(1)}\right)}{\ln(\lambda(\vphi))}$.
	Then, if $f$ satisfies the assumptions of Theorem \ref{thm:lowerbound},
	$$\PP(X_{k_0} \in S(w)) \geq \frac{1-\lambda(\vphi)}{32}.$$
\end{lemma}
\begin{proof}
	Note that if 	$U \sim \mathrm{Uniform}(\Spd)$, then 
	\begin{equation} \label{eq:uniformmixing}
	\PP(U \in S(w)) = \int\limits_{|t| \leq \frac{1-\lambda(\varphi)}{16\sqrt{d}}}\beta_d(t)dt  \geq 2\frac{\Gamma(\frac{d}{2})}{\Gamma\left(\frac{d-1}{2}\right)}\beta_d\left(\frac{1}{16\sqrt{d}}\right)\frac{1-\lambda(\varphi)}{16\sqrt{d}} \geq \frac{1-\lambda(\varphi)}{16}.
	\end{equation} 
	It will then suffice to show that $\PP(X_{k_0} \in S(w))$ can be well approximated by $\PP(U \in S(w))$. 
	Since $X_{k_0}$ has density $A^{k_0-1}_\vphi f(\inner{x}{x_0})$, the following holds true,
	\begin{align*}
	\left(\PP(U \in S(w)) - \PP(X_{k_0} \in S(w))\right)^2 &= \left(\int\limits_{S(w)}\left(A_\vphi^{k_0-1}f(\inner{x}{x_0}) - 1\right)d\sigma(x)\right)^2\\
	&\leq \int\limits_{S(w)}\left(A_\vphi^{k_0-1}f(\inner{x}{x_0}) - 1\right)^2d\sigma(x).
	\end{align*}
	We now decompose the density as $f(\inner{x}{x_0}) = \sum \limits_{i = 0}^\infty\lambda_i \psi_i(x)\psi_i(x_0)$. So that
	$$A_\vphi^{k_0-1}f(\inner{x}{x_0}) = \sum \limits_{i = 0}^\infty\lambda_i^{k_0} \psi_i(x)\psi_i(x_0).$$
	We know that $\psi_0 \equiv 1$ with eigenvalue $\lambda_0 = 1$, and we've assumed that $|\lambda_i|\leq \lambda_1 = \lambda(\varphi) $ for every $i \geq 1$. Thus, 
	\begin{align*}
	\int\limits_{S(w)}\left(A_\vphi^{k_0-1}f(\inner{x}{x_0}) - 1\right)^2d\sigma(x)
	&= \int\limits_{S(w)}\left(\sum\limits_{i =1}^\infty\lambda_i^{k_0}\psi_i(x)\psi_i(x_0)\right)^2d\sigma(x)\\
	&\leq(\lambda(\vphi))^{2k_0-2}\int\limits_{S(w)}\sum\limits_{i =1}^\infty(\lambda_i\psi_i(x))^2\sum\limits_{i =1}^\infty(\lambda_i\psi_i(x_0))^2d\sigma(x) \\
	&\leq \lambda(\vphi)^{2k_0-2}f(1)^2.
	\end{align*}
	where in the last inequality we have used $f(1) = \sum\lambda_i\psi_i(y)\psi_i(y)$, which is valid for any $y \in \Spd$.
	Thus, since $k_0 = \frac{\ln\left(\frac{\lambda(\vphi)(1-\lambda(\vphi))}{32 f(1)}\right)}{\ln(\lambda(\vphi))}$, by \eqref{eq:uniformmixing}, we get,
	$$\PP(X_{k_0} \in S(w)) \geq \PP(U \in S(w)) - \frac{1-\lambda(\vphi)}{32} \geq \frac{1-\lambda(\vphi)}{32}.$$
\end{proof}
Since the random walk $X_k$ mixes well, we now use Lemma \ref{lem:eigendps} to show that after enough steps, $X_k$ will be approximately invariant to a given reflection.
\begin{lemma} \label{lem:singlerelfection}
	Let $w,x_0 \in \Spd$. Then, 
	\[\mathrm{Law}((g(v))_{v \in T_k}|g(r) = x_0) \in \mathrm{DPS}^k_w\left(1- q^kp^{k}\right), 
	\]
	where $p = \left(1-\frac{\ln(\lambda(\vphi))}{\ln\left(\frac{\lambda(\vphi)(1-\lambda(\vphi))}{32 f(1)}\right)}\frac{(1-\lambda(\varphi))^2}{600}\right)$.
\end{lemma} 	
\begin{proof}
	Let $R =\mathrm{Id}_d - 2ww^T$ denote the linear reflection with respect to $w^\perp$. Then, the claim is equivalent to the decomposition,
	$$X_k = P\tilde{X}_k + (1-P)X_k^R,$$
	where $X_k^R$ is invariant to reflections by $R$ and $P \sim \mathrm{Bernoulli}(s_k)$ is independent from $\{\tilde{X}_k,X_k^R\}$ with  
 $s_k \leq  \left(1-\frac{\ln(\lambda(\vphi))}{\ln\left(\frac{\lambda(\vphi)(1-\lambda(\vphi))}{32 f(1)}\right)}\frac{(1-\lambda(\varphi))^2}{600}\right)^{k}$. \\

	Consider the case that for some $i = 0,...,k$, $|\langle X_i, w \rangle| \leq \frac{1-\lambda(\vphi)}{16\sqrt{d}}$. In this case, from Lemma \ref{lem:eigendps}, given $X_i$, we have the decomposition,
	$$\mu_\vphi(X_i) = \left(1 - \frac{(1-\lambda(\vphi))}{35}\right)\mu_{\vphi,w} + \frac{(1-\lambda(\vphi))}{35}\mu_{\vphi,w}^s(X_i),$$
	where $\mu_{\vphi,w}^s(X_i)$ is $R$-invariant. \\
	
	We now generate the random walk in the following way. For $i = 0,\dots,k$, let 
	\begin{equation} \label{eq:bernouli}
	P_i \sim \mathrm{Bernoulli}\left(\frac{(1-\lambda(\vphi))}{35}\right),
	\end{equation}
	be an \emph{i.i.d} sequence. Given $X_i$, if $|\langle X_i, w \rangle| > \frac{1-\lambda(\vphi)}{16\sqrt{d}}$ then $X_{i+1} \sim \mu_\vphi(X_i)$. Otherwise, $|\langle X_i, w \rangle| \leq \frac{1-\lambda(\vphi)}{16\sqrt{d}}$. To decide on the position of $X_{i+1}$ we consider $P_i$. If $P_i = 0$ then $X_{i+1} \sim \mu_{\vphi,w}(X_i)$. If $P_i = 1$, we generate $X_{i+1} \sim \mu_{\vphi,w}^s(X_i)$. We denote the latter event by $A_i$ and $A = \cup_{i=0}^{k-1}A_i$. \\
	
	It is clear that, conditional on $A$, $RX_k \stackrel{law}{=} X_k$. Thus, to finish the proof, if $\bar{A}$ is the complement of $A$, we will need to show 
	$$\PP(\bar{A}) \leq \left(1-\frac{\ln(\lambda(\vphi))}{\ln\left(\frac{\lambda(\vphi)(1-\lambda(\vphi))}{32 f(1)}\right)}\frac{(1-\lambda(\varphi))^2}{600}\right)^{k}.$$
	Towards this, let $S(w)$ and $k_0$ be as in Lemma \ref{lem:mixing}. Coupled with \eqref{eq:bernouli}, the lemma tells us that
	$$\PP\left(A_{k_0}\right) \geq \frac{(1-\lambda(\vphi))^2}{600}.$$
	Now, by restarting the random walk from $X_{k_0}$ if needed, we may show,  
	$$\PP\left(\bar{A}\right)\leq \sum\limits_{m\leq \frac{k}{k_0}}\PP(\bar{A}_{m\cdot k_0})  \leq \left(1- \frac{(1-\lambda(\vphi))^2}{600}\right)^{\frac{k}{k_0}} \leq \left(1- \frac{(1-\lambda(\vphi))^2}{600k_0}\right)^{k}.$$
	The claim now follows by taking a union bound over all paths.
	Indeed, for each $v \in T_k$ define the corresponding event $A_v$. Then, with a union bound,
	$$\PP\left(\bigcup_{v \in T_k}\bar{A_v}\right) \leq |T_k|p^k \leq q^kp^k.$$
	By definition, conditioned on $\bigcap_{v\in T_k}{A_v}$,  $\{X_v\}_{v\in T_k}$ is symmetric around $w$, which completes the lemma.
\end{proof}
\subsubsection{Proving Theorem \ref{thm:lowerbound}}
\begin{proof}
	By Lemma \ref{lem:singlerelfection}, for every $w,x_0 \in \Spd$.
	\[ \mathrm{law}(g(v)_{v\in T_k}|g(r) = x_0) \in \mathrm{DPS}^k_w(1-q^kp^k), \]
	where $p = \left(1-\frac{\ln(\lambda(\vphi))}{\ln\left(\frac{\lambda(\vphi)(1-\lambda(\vphi))}{32 f(1)}\right)}\frac{(1-\lambda(\varphi))^2}{600}\right)$.
	By assumption
	\[ q \leq p^{-1}, \]
	and Lemma \ref{lem:generallower} gives the result.
\end{proof}
\paragraph{Acknowledgments:} We are grateful to two anonymous referees for carefully reading the paper and providing many insightful comments.
\bibliographystyle{abbrv}
\bibliography{SBM}
\end{document}